\documentclass[letterpaper]{article} 
\usepackage{aaai25}  
\usepackage{times}  
\usepackage{helvet}  
\usepackage{courier}  
\usepackage[hyphens]{url}  
\usepackage{graphicx} 
\usepackage{natbib}  
\usepackage{caption} 
\usepackage{algorithm}
\usepackage{algorithmic}

\usepackage{amsmath}
\usepackage{amssymb}
\usepackage{amsthm}
\newtheorem{lemma}{Lemma}
\newtheorem{definition}{Definition}

\newcommand{\stitle}[1]{\paragraph{\textit{{#1.}}}}
\usepackage{newfloat}
\usepackage{listings}
\DeclareCaptionStyle{ruled}{labelfont=normalfont,labelsep=colon,strut=off} 
\floatstyle{ruled}
\newfloat{listing}{tb}{lst}{}
\floatname{listing}{Listing}
\title{Tokenphormer: Structure-aware Multi-token Graph Transformer \\ for Node Classification}
\author{
    Zijie Zhou\textsuperscript{\rm 1}\equalcontrib,
    Zhaoqi Lu\textsuperscript{\rm 1, \rm 2, \rm 3}\equalcontrib,
    Xuekai Wei\textsuperscript{\rm 1}, 
    Rongqin Chen\textsuperscript{\rm 1}, 
    Shenghui Zhang\textsuperscript{\rm 1}, 
    Pak Lon Ip\,\textsuperscript{\rm 1}, 
    Leong Hou U\textsuperscript{\rm 1}\thanks{Corresponding author.}
}
\affiliations{
    \textsuperscript{\rm 1}University of Macau\\
    \textsuperscript{\rm 2}The Hong Kong Polytechnic University\\
    \textsuperscript{\rm 3}Guangdong Institute of Intelligent Science and Technology, China\\

    \{zhou.zijie, lu.zhaoqi\}@connect.um.edu.mo; xuekaiwei2-c@my.cityu.edu.hk; \{chen.rongqin, zhang.shenghui\}@connect.um.edu.mo; \{paklonip, ryanlhu\}@um.edu.mo
}

\usepackage{bibentry}

\begin{document}

\maketitle

\begin{abstract}
Graph Neural Networks (GNNs) are widely used in graph data mining tasks. Traditional GNNs follow a message passing scheme that can effectively utilize local and structural information. However, the phenomena of over-smoothing and over-squashing limit the receptive field in message passing processes. Graph Transformers were introduced to address these issues, achieving a global receptive field but suffering from the noise of irrelevant nodes and loss of structural information. Therefore, drawing inspiration from fine-grained token-based representation learning in Natural Language Processing (NLP), we propose the Structure-aware Multi-token Graph Transformer (Tokenphormer), which generates multiple tokens to effectively capture local and structural information and explore global information at different levels of granularity. Specifically, we first introduce the walk-token generated by mixed walks consisting of four walk types to explore the graph and capture structure and contextual information flexibly. To ensure local and global information coverage, we also introduce the SGPM-token (obtained through the Self-supervised Graph Pre-train Model, SGPM) and the hop-token, extending the length and density limit of the walk-token, respectively. Finally, these expressive tokens are fed into the Transformer model to learn node representations collaboratively. Experimental results demonstrate that the capability of the proposed Tokenphormer can achieve state-of-the-art performance on node classification tasks.
\end{abstract}

%
\begin{links}
    \link{Code}{https://github.com/Dodo-D-Caster/Tokenphormer}
    \link{Appendix}{https://doi.org/10.48550/arXiv.2412.15302}
\end{links}

\section{Introduction}\label{introduction}
Graph structures are commonly seen in both physical and virtual domains, such as anomaly detection~\cite{aaai21AnomalyDetection}, traffic~\cite{aaai24Traffic}, and recommendation~\cite{aaai24Recommendation}. Mining tasks face a significant challenge due to the inherent complexity of various graphs and their intricate features. Graph Neural Networks (GNNs) have emerged as a successful learning framework capable of solving various graph-based tasks. Traditional GNNs follow the message passing scheme~\cite{GraphSAGE_Hamilton2017InductiveRL}, aggregating neighboring nodes to update node representation. However, these methods are hindered by issues such as over-smoothing~\cite{Chen2020} and over-squashing~\cite{Alon_Yahav_2020}.

\begin{figure}[tb!]
    \centering
    \includegraphics[width=0.47\textwidth, trim=0 250 0 0, clip]{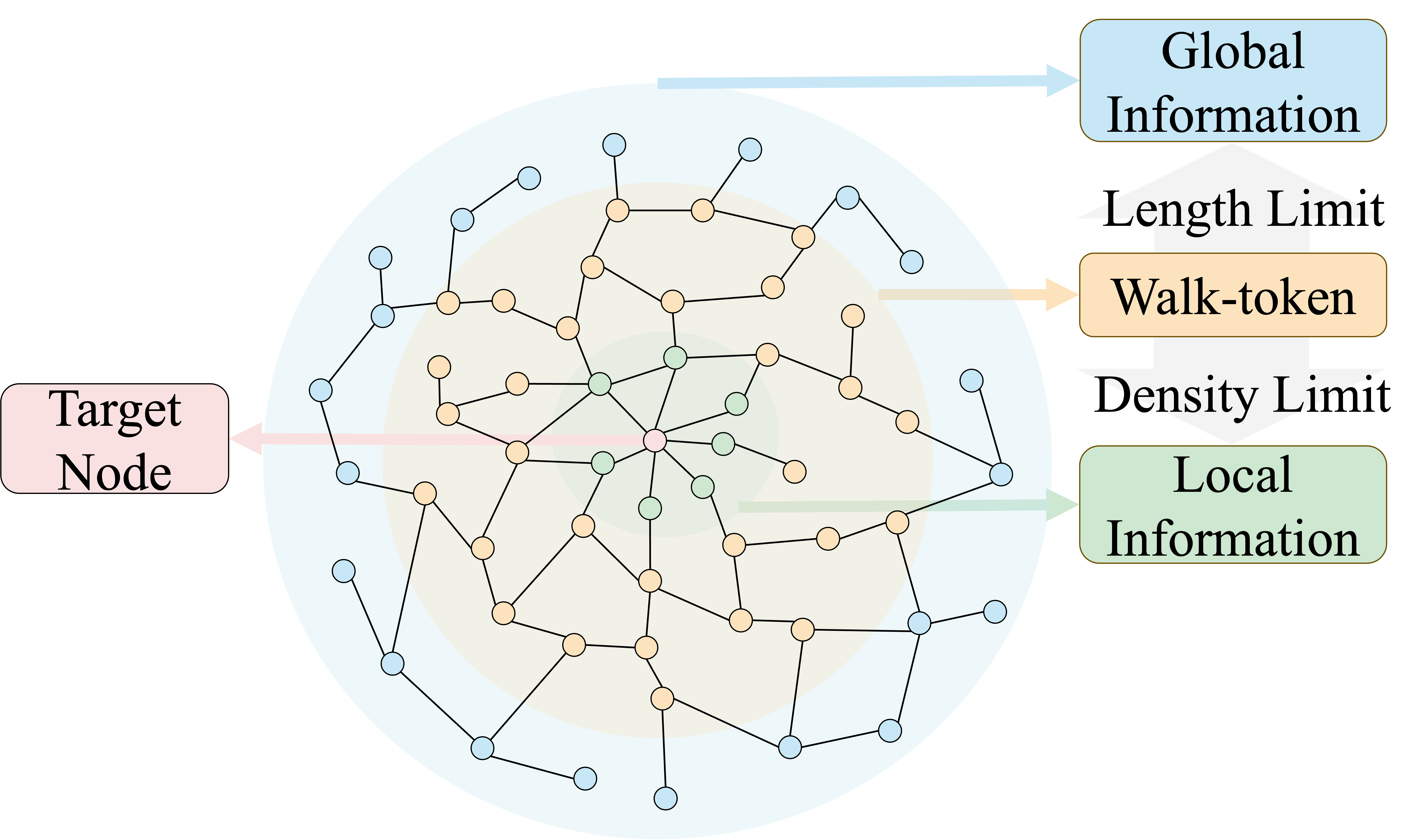}
    \caption{Idea of Tokenphormer. 
    }
    \label{motivation}
\end{figure}

Transformer~\cite{Vaswani_Shazeer_Parmar_Uszkoreit_Jones_Gomez_Kaiser_Polosukhin_2017} has demonstrated remarkable performance due to its powerful modeling capability, making it a potent architecture for graphs as well. The traditional graph Transformer treats the entire input graph as a fully connected entity, where individual nodes are regarded as independent tokens, and a unified sequence encompassing all nodes is constructed, thus alleviating the limitations of Message Passing Neural Networks (MPNNs). This approach allows the graph Transformer to extend the receptive field to the global, which benefits node representation learning. However, this scheme also results in high computational complexity, introduces noise from irrelevant nodes, and leads to a loss of structural information~\cite{Kreuzer_Beaini_Hamilton_Létourneau_Tossou_2021}. This limitation potentially hinders a comprehensive understanding of the entire graph. 

In recent years, researchers have developed token-based methods~\cite{Kreuzer_Beaini_Hamilton_Létourneau_Tossou_2021, Gophormer, ANS-GT} that generate a sequence for each node
, which have shown improved performance compared to the traditional graph Transformer. However, these models encounter difficulties when dealing with various graphs, as they lack the flexibility to thoroughly {\em explore the complete graph structure} due to the monotonous token design. For instance, NAGphormer~\cite{NAGphormer}, which serves as a representative and has achieved competitive results on homogeneous graphs, utilizes hop information to form tokens. These tokens are coarse-grained since they overlook crucial information, such as relationships among nodes within the same hop and data beyond the designated $k$-hop neighborhood, limiting the model's performance on diverse graph types like heterogeneous graph.

To address the aforementioned limitations, token generation strategies should be enhanced in two key aspects: 1) \textbf{Fine-grained.} Given the diverse and complex nature of graph structures, tokens should be more fine-grained, effectively capturing intricate structures and node relationships. 2) \textbf{Comprehensive.} To gain a better understanding of the entire graph, tokens should provide an approximate full coverage, capturing both global and local information for improved expressiveness. A natural question arises: {\em Can we design token generation strategies that can create multiple effective tokens for graph nodes that substantially improve the ability to investigate varied graph structures?}

To answer this question, motivated by fine-grained token-based representation learning in Natural Language Processing (NLP), we propose {\em graph serialization}, which transforms the graph structure into {\em graph document} consisting of {\em graph sentences} using random walk, to effectively serialize the entire graph as well as node's neighborhood with minimal information loss. Graph serialization is the foundation for designing multiple token generators, enabling a more flexible exploration of diverse graph types.

Figure~\ref{motivation} shows the idea of Tokenphormer. To meet the fine-grained requirement, we utilize {\em walk-token} to flexibly explore the graph (yellow area in Figure~\ref{motivation}). Specifically, the {\em walk-token} incorporates four distinct walk types, each representing a serialized node's neighborhood with varying receptive fields and walking strategies. These walk types are designed to capture a diverse range of information, making them suitable for various graph structures. The four types are: uniform random walk, non-backtracking random walk, our proposed neighborhood jump walk, and non-backtracking neighborhood jump walk. It should be noted that each walk type is non-irreplaceable, just as a tailored design of data augmentation to address the requirement of fine-grained.

Given the limited receptive field of {\em walk-token}, we address the comprehensive requirement by introducing two additional token types: {\em SGPM-token} (blue area in Figure~\ref{motivation}) and {\em Hop-token} (green area in Figure~\ref{motivation}), extending the length and density limit of {\em walk-token}. Specifically, {\em SGPM-tokens} are obtained from the Self-supervised Graph Pre-train Model (SGPM), which pre-trains on the graph document (serialized entire graph), ensuring the global coverage of the {\em walk-token}. Moreover, {\em Hop-tokens} are generated through each hop's message, ensuring the local coverage of the {\em walk-token}.

Finally, we introduce Tokenphormer (Structure-aware Multi-token Graph Transformer),
a node classification approach that leverages the power of Transformer to jointly learn all these token types. Tokenphormer pioneers {\em the generation of diverse fine-grained and comprehensive tokens for data augmentation, adaptly extracting information across a spectrum of scales, from local to global contexts} and is adaptive to diverse graphs, outperforming other methods on both homogeneous and heterogeneous graphs.
The main contributions of this work are as follows:
\begin{itemize}   
  \item  We present Tokenphormer, a pioneering node classification approach that generates multiple tokens rich in structural information and varying granularities through diverse walk strategies, to address the limitations of existing graph Transformers.
  \item We propose an effective way to serialize the graph into a graph document and introduce SGPM to capture global and contextual information from our proposed graph documents in the pre-train period.
  \item Experimental results demonstrate that Tokenphormer outperforms existing state-of-the-art graph Transformers and mainstream MPNNs on most homogeneous and heterogeneous graph benchmark datasets, showcasing its applicability to diverse graphs. 
\end{itemize} 



\section{Related Work}\label{RelatedWork}
\subsection{Existing Graph Learning Architectures} \label{GNNs}

GNN methods can be generally categorized nto  Message Passing Neural Networks (MPNNs) and graph Transformers. The MPNN approach involves two main steps: aggregation and update. Each node first aggregates features from its neighbors and then updates its representation by combining its features with the aggregated data~\cite{MPNN1_pmlr-v80-xu18c, MPNN2_pmlr-v70-gilmer17a}. Notable models like GCN~\cite{GCN_kipf2017semi}, GAT~\cite{GAT_veličković2018graph}, and GraphSAGE~\cite{GraphSAGE_Hamilton2017InductiveRL} have demonstrated strong performance with this method. However, MPNNs face challenges such as over-smoothing~\cite{Chen2020} and over-squashing~\cite{Alon_Yahav_2020}. To address these issues, methods like DropEdge~\cite{dropedge}, which randomly removes edges during training, and techniques to reduce message passing redundancy~\cite{NEURIPS2022_RFGNN}, have been developed.

The Transformer architecture~\cite{Vaswani_Shazeer_Parmar_Uszkoreit_Jones_Gomez_Kaiser_Polosukhin_2017, schmitt-etal-2021-modeling, dufter2022position, GOAT, largeGT} has gained popularity in graph representation learning as a solution to these problems. For example, GT~\cite{GT} uses Laplacian eigenvectors for positional encoding, while GraphTrans~\cite{GraphTrans} and GROVER~\cite{GROVER} incorporate GNNs as auxiliary modules to capture structural information. Graphormer~\cite{Graphormer} integrates centrality and spatial encoding to add structural inductive bias into the attention mechanism. These methods treat the entire graph as a sequence of node tokens and are of high computation complexity. Recent advancements like Gophormer~\cite{Gophormer} sample relevant nodes to form token sequences, shifting the Transformer's training focus from the entire graph to node sequences. NAGphormer~\cite{NAGphormer} introduces the Hop2Token mechanism, converting neighborhood features from each hop into different token sequences. However, though the complexity is decreased, the simplex design of tokens limits these models' performance on diverse graphs.

\begin{figure*}[!hbt]
\centering
\includegraphics[width=0.85\textwidth]{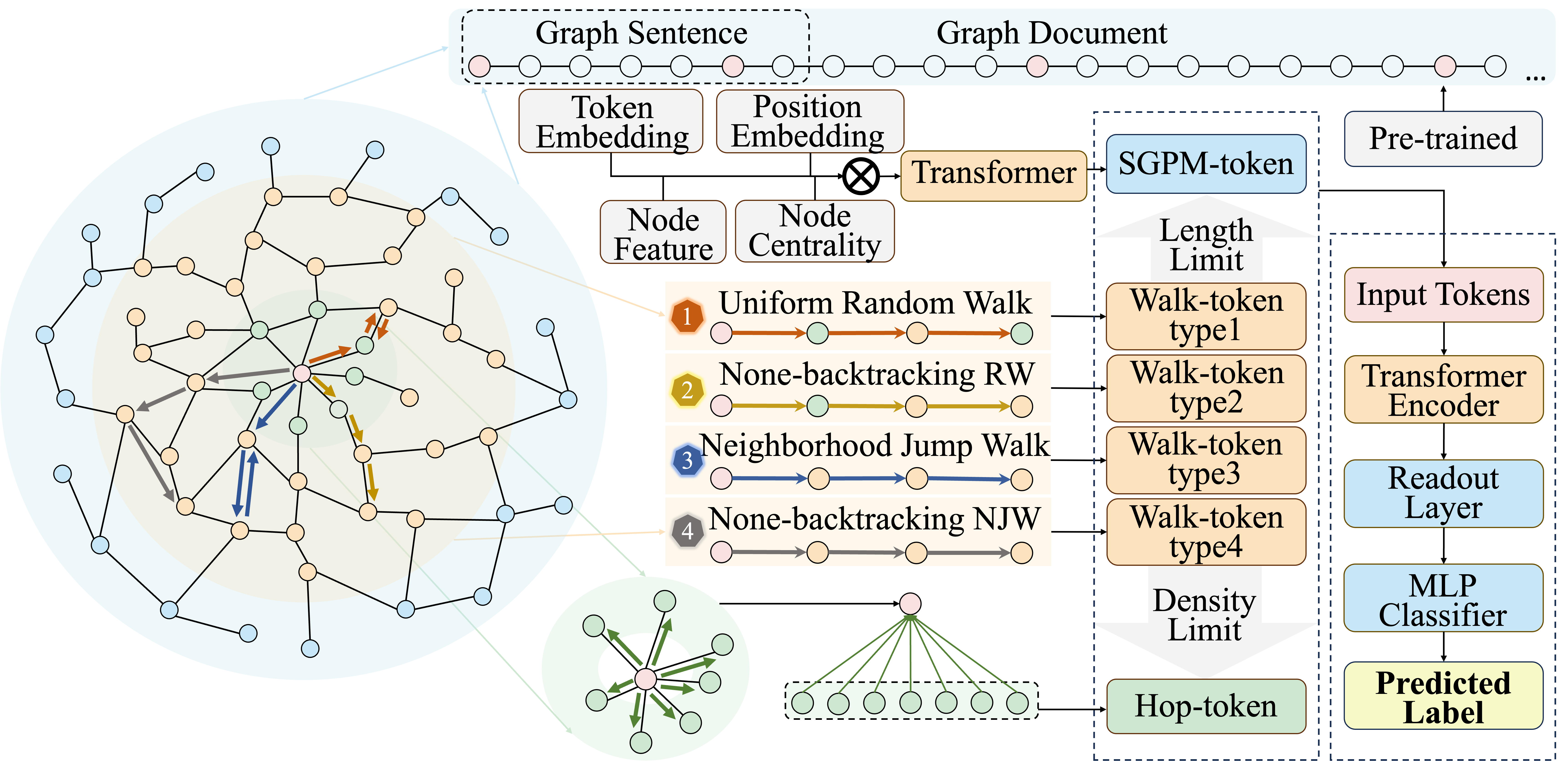}
\caption{Framework. RW refers to random walk while NJW stands for neighborhood jump walk. Tokenphormer generates diverse tokens with different levels of granularity for the target node (red node), respectively {\em walk-tokens} (yellow area), {\em SGPM-token} (blue area) and {\em hop-token} (green area), comprehensively mining essential information from the whole graph. Then, all these tokens are constructed into a sequence and fed into the Transformer-based backbone to jointly learn the final node representation. Finally, an MLP-based module is employed for node classification tasks.
}
\label{tokenphormer}
\end{figure*}

\subsection{Token-based Representations}\label{NLP}

NLP has advanced significantly due to deep learning and large language models (LLMs) like GPT~\cite{GPT}. These improvements hinge on sophisticated token representations that capture semantic and syntactic nuances, enabling applications like sentiment analysis and machine translation. Techniques like Word2Vec~\cite{Word2Vec} and GloVe~\cite{GloVe} use neural networks to place words in continuous vector spaces, enhancing contextual understanding. Models like BERT~\cite{BERT} and GPT~\cite{GPT} further improve this with contextualized embeddings that consider surrounding words, leading to nuanced semantic comprehension.

In pursuing advanced token representations, there is a shift towards enhanced tokens within graph nodes~\cite{Graph-Bert}. Traditional graph Transformers~\cite{GT_yun2019GTN, Graphormer} use independent nodes as tokens, resulting in high computational complexity and increased noise. Node2Sequence methods like Gophormer~\cite{Gophormer} sample ego-graphs and add global nodes to form sequences but still overlook edge information, limiting node representation. The Hop2Token method proposed by NAGphormer~\cite{NAGphormer} aggregates each hop into a token but remains coarse-grained. To make the input tokens of the graph Transformers more fine-grained, inspired by subword and character-level tokenization~\cite{Word2Vec}, generating multiple tokens for nodes can better capture graph structure details, nuances, and semantic intricacies, aligning with the trend of comprehensive token representations to enhance graph models.

\section{The Proposed Method}\label{ProposedMethod}




\subsection{Notations and Problem Formulation}
\label{notationsAndProblemFormulation}
Consider $G=(V, E)$ be an unweighted graph with node set $V$ and edge set $E$, where $V = \{{v_1, v_2, \cdots, v_n}\}$ and $n=|V|$. Each node's feature vector $x_v \in X$, where $X \in \mathbb{R}^{n \times d_F}$ is a feature matrix with $d_F$ dimensions to describe the attribute information of nodes. The adjacency matrix of $G$ is denoted by $A \in \mathbb{R}^{n \times n}$. A random walk in the graph $G$ follows a transition probability matrix $P$, which varies depending on the type of walk. Here, $P_{ij}$ denotes the probability of node $i$ choosing node $j$ as its next state. In this paper, we focus mainly on the semi-supervised node classification task. Specifically, given nodes in the training set $T_V$ with known labels $Y_v$ and feature vectors $X_v$ for $v \in V$, our aim is to predict the unknown $Y_u$ for all $u \in (V - T_V )$.

\subsection{Preliminaries}\label{preliminary}

\stitle{Markov chain and graph random walk}
A Markov chain~\cite{chung1967markov} is a stochastic process where transitions between states depend only on the current state, reflecting its memoryless property. It has the recurrence property if it can revisit certain states. A {\em random walk}~\cite{lovasz1993random} on a graph is a type of Markov chain, where each step depends on the current vertex.

\stitle{Central limit theorem in graphs}
The central limit theorem for Markov chains states that if two independent chains with arbitrary initial distributions share the same transition matrix, their limiting distributions will converge to the same value as time approaches infinity~\cite{10.1214/154957804100000051}. 

The limiting distribution can be reached by a Markov chain with {\em recurrence} and {\em aperiodic} property, satisfying:
\begin{equation}
\lim_{n \to \infty} (S_n = s_i) = \pi (s_i)
\end{equation}
where $S_n$ is the $n$-th state of the Markov chain and $s_i$ is the $i$-th state in the state space. A random walk on a connected and non-bipartite graph is recurrent and aperiodic, guaranteeing convergence to the limiting distribution and satisfying the central limit theorem, as formalized in the following lemma:
\begin{lemma}\label{lemma1}
If $G$ is a connected, non-bipartite graph, then for any initial distribution $\pi_0$ on $v \in V$, we have:
\begin{equation}
    \lim_{n \to \infty} (\pi_0 P^n)(v) = \pi(v)
\end{equation}
\end{lemma}

\subsection{Graph Serialization}\label{Graph-Serialization}

Graph serialization encompasses transforming a graph structure into a sequence structure
and there exist two cases:

\stitle{Case 1: Serializing a graph by a long walk} 
Leveraging the principles of the central limit theorem (Lemma~\ref{lemma1}) in the context of graphs offers a viable strategy for graph serialization. When the length of a random walk on a graph tends towards infinity, the initial distribution's impact diminishes, leading to a more comprehensive depiction of the underlying graph structure. We name the document generated through this method the {\em Graph Document}, defined as follows:

\begin{definition}[Graph sentence and document]
\label{graphDocument}
Graph sentence refers to an individual walk in the graph; the collection of a large number of graph sentences starting from every node in the graph constitutes a graph document.
\end{definition}

Nevertheless, it is crucial to acknowledge that employing an infinite-length walk on a graph is impractical during training. In practice, the limiting distribution can be approximated through a finite number of random walks, each with a limited length. 
Let $walk(v_i, v_j)$ denote a walk starting from node $v_i$ and ending at node $v_j$. Then $walk(v_i, v_j)$ and $walk(v_j, v_k)$ can be viewed as partitions of $walk(v_i, v_k)$. By concatenating the graph sentences, we can create significantly longer walks that encompass all the nodes in the graph.

\stitle{Case 2: Serializing nodes by diverse walks} Despite serializing the whole graph into graph document, it is also important to serialize the node's neighborhood for node classification tasks. 
To ensure that all sentences in the graph are strongly connected with the target node, the graph serialization approach can involve multiple diverse random walks initiated from the target node.
Moreover, We also demonstrate that the coverage speed of a node's neighborhood increases exponentially through this neighborhood serialization method in Appendix A.2.

\subsection{Graph Tokenization}
Based on graph serialization, graph tokenization tokenizes the graph into various token types to meet the requirement of fine-grained and comprehensive node representation.


\subsubsection{Walk-token}\label{walk_token}
Walk-tokens aim to achieve fine-grained representation, aligning with the second case of graph serialization. They are generated through four distinct walk types, each capturing information from nearby and distant neighbors at varying granularity levels.


\stitle{(1) Uniform and non-backtracking random walk}
{\em Uniform random walk} randomly selects the next node from the current node's neighbors. Additionally, we explore non-backtracking random walk, which exhibits faster convergence to its limiting distribution compared to uniform random walk in most cases~\cite{ALON_BENJAMINI_LUBETZKY_SODIN_2007}.

\begin{definition}[Non-backtracking random walk]
\label{def2}
    A non-backtracking random walk on G is a sequence $(v_0, v_1, \ldots, v_k)$ of vertices $v_i \in V$ where $v_{i+1}$ is randomly selected from the neighbors of $v_i$, and it must satisfy the condition $v_{i+1} \ne v_{i-1}$ for $i=1,\ldots,k-1$.
\end{definition}

The non-backtracking random walk avoids revisiting the last encountered node. This seems to contradict the Markov property. However, by redefining the state space from nodes to edges, the non-backtracking random walk can still be seen as a Markov chain~\cite{Non-backtracking}, allowing us to use the earlier mentioned theorem. 

Let the edge connecting node $u$ to node $v$ be denoted as $(u,v)$. The transition probability matrix $\tilde P((u,v),(x,y))$ for the non-backtracking random walk can be expressed as: 
\begin{equation}
    \tilde P((u,v),(x,y))  =\left\{
    \begin{array}{rcl}
    \displaystyle{\frac{1}{d_v-1}}       &      & {if \;  v=x, \; y\ne u}\\
    0                                    &      & {otherwise}\\
    \end{array} \right.
\end{equation}
where $d_v$ denotes the degree of node $v$.

Furthermore, non-backtracking random walks can traverse longer distances and reduce redundancy in the generated walk sequences~\cite{NEURIPS2022_RFGNN}.

\stitle{(2) Neighborhood jump walk and non-backtracking version} 
Conventional random walk approaches are limited by walk length. While the information from nearby neighbors is more significant, disregarding distant neighbors is not advisable. Inspired by dilated convolution~\cite{dilatedConv}, which has been proved to be effective~\cite{dilatedConvProof}, we define the neighborhood jump walk as a random walk with dilated transition probability. Unlike uniform random walks, the neighborhood jump walk can jump to any node within its $k$-hop neighborhood and continue recursively. Thus, the neighborhood jump walk's reachable area expands by the neighborhood's radius. Below is the definition of the neighborhood jump walk:

\begin{definition}[Neighborhood jump walk]
\label{def3}
    A neighborhood jump walk on G is a sequence $(v_0, v_1, ..., v_k)$ of vertices $v_i \in V$ where $v_{i+1}$ is chosen among the nodes in the $k$-hop neighborhood of $v_i$. The transition probability of $v_i$ to choose $v_{i+1}$ in $v_i$'s $k$-hop neighborhood follows $k$-step probability propagation.
\end{definition}


\begin{figure}[!t]
\centering
\includegraphics[width=0.49\textwidth]{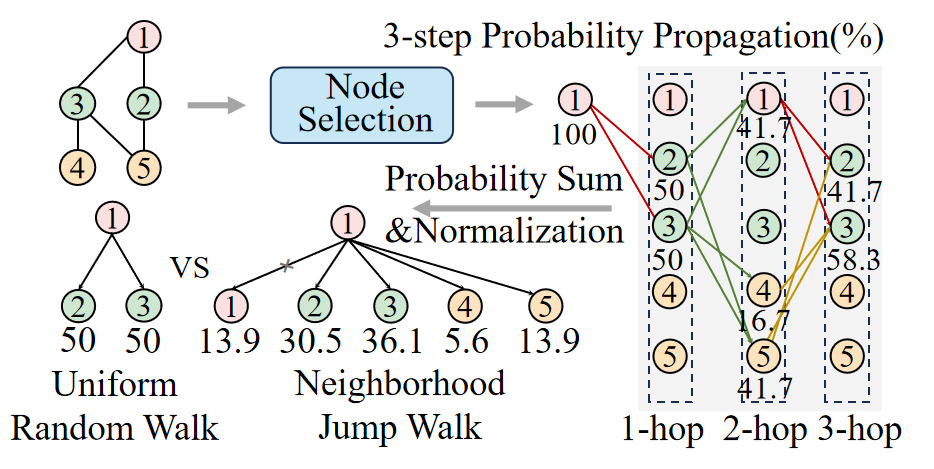}
\caption{Transition Probability of Neighborhood Jump Walk.}
\label{tp_neighborhoodjump}
\end{figure}

The transition probability matrix for the neighborhood jump walk, described in Definition \ref{def3}, is computed using a $k$-step probability propagation process. In the given graph (Figure~\ref{tp_neighborhoodjump}), for the $1$-st step, the probability of node 1 (100\%) is split averagely to its neighbors, node 2 (50\%) and node 3 (50\%). This process iterates to calculate the likelihood of node 1 reaching each node within its 3-hop neighborhood. By aggregating and normalizing these probabilities, the overall probability of node 1 visiting any node in its 3-hop neighborhood is established. 

These probabilities can be calculated in advance, thereby preserving the Markov property.
Moreover, to prevent a node from revisiting itself, the probability of self-loops is eliminated. The likelihood of returning to the last visited node can also be negated to embed the {\em non-backtracking property}, which can accelerate the walk's convergence. 

\subsubsection{SGPM-token}\label{sgpm-token}
SGPM-token extends the length limit of walk-tokens, aligning with the first case of graph serialization and facilitating a comprehensive exploration of positional nuances across nodes in the graph. SGPM-tokens are created by applying weights derived from the Self-supervised Graph Pre-trained Model (SGPM). 

\stitle{Self-supervised Graph Pre-trained Model, SGPM}
Self-supervised pre-trained models have emerged as a highly effective approach for learning representations from unlabeled NLP data. Inspired by this, we introduce SGPM to learn from graph document unsupervisedly. 

\stitle{(1) Document generation}
As stated in Definition \ref{graphDocument}, a graph document consists of numerous graph sentences, which are walks generated by individual nodes with varying lengths. In the pre-training phase of SGPM, preserving the raw node connections is crucial for learning the graph's structure and contextual information. Thus, we use non-backtracking random walks as our final choice, as they extend the walk while maintaining original relationships. The walk length follows a normal distribution, $\mathcal{N}(\mu, \sigma^2)$, where the mean $\mu$ corresponds to the graph's radius $R$ and the standard deviation $\sigma$ is set to 1.

\stitle{(2) Tokenizer and input representation}
After generating the graph document, the next step involves tokenizing the sentences. Similar to tokenization in NLP, a {\em vocabulary} is created to encompass all tokens. In graph structures, each node maps directly to a corresponding token. Additionally, five special tokens are added: \texttt{PAD} for sequence padding, \texttt{UNK} for unknown nodes, \texttt{CLS} and \texttt{SEP} for sequence boundaries, and \texttt{MASK} for replacing hidden nodes.

Once the vocabulary is established, the graph sentences are tokenized into sequences comprised of these tokens. Before input into the SGPM, input representations for each token in the sequence are needed. Notably, let \( S \) denote the set of tokens, and for each token $j \in S$, the token embeddings $T_{j}$ and position embeddings $P_{j}$ are added, which are commonly used in NLP pre-train models. Furthermore, our input representation $h_{j}$ is enriched by incorporating node features $X_{j}$ and measures of node centrality $C_{j}$. Finally, we fuse these embeddings using sum pooling:
\begin{align}
    h_{j} = \text{SUM}(T_{j}, P_{j}, X_{j}, C_{j})
\end{align}
\stitle{(3) Graph pre-training process}
Based on experience from NLP, we employ the Masked Language Model (MLM)~\cite{BERT} as SGPM's pre-train task. Let \( \mathbf{H} = [h_1, h_2, \ldots, h_N] \) represent the input sequence of tokens, and \( N \) the total number of tokens. Initially, 15\% of the tokens are randomly replaced with the special token \texttt{MASK}. This masked sequence is then input into the bidirectional Transformer model, which generates individual scores \( \mathbf{S} = [s_1, s_2, \ldots, s_N] \) for each token. These scores represent the likelihood of each token given its context and are then compared to the actual ground truth values \( \mathbf{Y} = [y_1, y_2, \ldots, y_N] \), where \( y_i \) is the true label of the masked token \( h_i \). The cross-entropy loss is calculated as:
\begin{equation}
    \mathcal{L} = -\frac{1}{N} \sum_{i=1}^{N} y_i \log(s_i)
\end{equation}

\subsubsection{Hop-token}\label{hop_toekn}
{\em Hop-token} is designed to capture hop information within a limited $k$-hop neighborhood, ensuring a structured and comprehensive representation of a node's local context. In Tokenphormer, the $k$-th {\em Hop-token} of the target node is computed as decoupled $(k+1)$-th layer of message passing, as depicted in Fig \ref{tokenphormer}. This approach allows us to incorporate the information from nodes within the $k$-hop radius, providing a rich representation of the node's immediate surroundings.

\begin{table*}[htb!]
    \begin{center}
    \small
    \setlength{\tabcolsep}{1mm}{%
    \begin{tabular}{lccccccc}
    \hline
    \textbf{Method} &\textbf{Year} &\textbf{Cora} &\textbf{Citeseer} &\textbf{Flickr} &\textbf{Photo} &\textbf{DBLP} &\textbf{Pubmed}\\
    \hline
    GCN & 2017 & 87.33$\pm$0.38 & 79.43$\pm$0.26 & 61.49$\pm$0.61 & 92.70$\pm$0.20 & 83.62$\pm$0.13 & 86.54$\pm$0.12\\
    GAT & 2017 & 86.29$\pm$0.53 & 80.13$\pm$0.62 & 54.29$\pm$2.56 & 93.87$\pm$0.11 & 84.19$\pm$0.19 & 86.32$\pm$0.16\\
    GraphSAGE & 2018 & 86.90$\pm$0.84 & 79.23$\pm$0.53 & 60.37$\pm$0.27 & 93.84$\pm$0.40 & 84.73$\pm$0.28 & 86.19$\pm$0.18\\
    APPNP & 2018 & 87.15$\pm$0.43 & 79.33$\pm$0.35 & \textbf{93.25$\pm$0.24} & 94.32$\pm$0.14 & 84.40$\pm$0.17 & 88.43$\pm$0.15\\
    JKNet & 2018 & 87.70$\pm$0.65 & 78.43$\pm$0.31 & 53.66$\pm$0.40 & - & 84.57$\pm$0.28 & 87.64$\pm$0.26\\
    GPR-GNN & 2021 & 88.27$\pm$0.40 & 78.46$\pm$0.88 & - & 94.49$\pm$0.14 & - & 89.34$\pm$0.25\\
    GATv2 & 2022 & - & - & - & - & 84.96$\pm$0.47 & 85.75$\pm$0.55\\
    GRAND+ & 2022 & 85.8x$\pm$0.4x &  75.6x$\pm$0.4x & - & 94.75$\pm$0.12 & - & 88.64$\pm$0.09\\
    \hline
    GT & 2020 & 71.84$\pm$0.62 & 67.38$\pm$0.76 & 68.59$\pm$0.64 & 94.74$\pm$0.13 & 81.04$\pm$0.27 & 88.79$\pm$0.12\\
    Graphormer & 2021 & 72.85$\pm$0.76 & 66.21$\pm$0.83 & 66.16$\pm$0.24 & 92.74$\pm$0.14 & 80.93$\pm$0.39 & 82.76$\pm$0.24\\
    SAN & 2021 & 74.02$\pm$1.01 & 70.64$\pm$0.97 & 70.26$\pm$0.73 & 94.86$\pm$0.10 & 83.11$\pm$0.32 & 88.22$\pm$0.15\\
    Gophormer & 2021 & 87.85$\pm$0.10 & 80.23$\pm$0.09 & 91.51$\pm$0.28 & - & 85.20$\pm$0.20 & 89.40$\pm$0.14\\
    ANS-GT & 2022 & 88.60$\pm$0.45 & 80.25$\pm$0.39 & - & 94.41$\pm$0.62 & - & 89.56$\pm$0.55 \\
    GraphGPS & 2022 & - & - & - & 95.06$\pm$0.13 & - & 88.94$\pm$0.16\\
    Exphormer & 2023 & - & - & - & 95.27$\pm$0.42 &- & 89.52$\pm$0.54 \\
    Gapformer & 2023 & 87.37$\pm$0.76 & 76.21$\pm$1.47 & - & 94.81$\pm$0.45 & \textbf{85.50$\pm$0.43} & 88.98$\pm$0.46 \\
    NAGphormer & 2023 & 90.56$\pm$0.39 & 80.02$\pm$0.80 & 89.66$\pm$0.63 & 95.49$\pm$0.11 & 84.62$\pm$0.13 & 89.60$\pm$0.14\\
    \hline
    Tokenphormer & ours &  \textbf{91.20$\pm$0.47} &  \textbf{81.04$\pm$0.24} & 92.44$\pm$0.35 &  \textbf{96.14$\pm$0.14} & 85.13$\pm$0.10 &  \textbf{89.94$\pm$0.20}\\
    \hline
    \end{tabular}
    }
    \end{center}
    
    \caption{Comparison of Tokenphormer with baselines on various datasets. The best results are in \textbf{bold}. `x' and `-' mean unknown numbers. Part values of NAGphormer are run by ourselves due to missing of standard deviation in raw paper~\cite{NAGphormer}.}
    \label{tab:main}
\end{table*}

\subsection{Tokenphormer}\label{Multi-tokenContextualGraphTransformer}
Tokenphormer aims to utilize diverse tokens to jointly learn node embeddings for node classification. As illustrated in Figure~\ref{tokenphormer}, each node $v$ has one SGPM-token, $n$ hop-tokens, and $m$ walk-tokens. These tokens are sequentially input into the Transformer encoder for processing, which includes multi-head self-attention (MSA) and position-wise feed-forward network (FFN) components, to compute the output of the $l$-th Transformer layer $H_{v}^{l}$:
\begin{equation}
    \begin{split}
        &\hat{H}_{v}^{l} = MSA(Norm(H_{v}^{l-1}))  + H_{v}^{l-1} \\
        &H_{v}^{l} = FFN(Norm(\hat{H}_{v}^{l}) + \hat{H}_{v}^{l}
    \end{split}
\end{equation}
where $l = 1, \ldots , L$ denotes the $l$-th Transformer layer. Within each mini-batch of inputs, we apply Layer Normalization using the $Norm(\cdot)$ function.

Since tokens are generated using different strategies, they may have varying contributions to the target node embedding. Instead of employing commonly used readout functions such as mean and summation, we employ an attention-based readout function to learn the different weights of token embeddings. Let $H \in R^{K \times d_{h}}$ denote the token representation of a node. Here, $K$ represents the total number of tokens, $d_h$ is the hidden dimension of the Transformer, and $H_k$ is the $k$-th token representation of the node. We calculate the attention parameter between different tokens by:
\begin{equation}
    \alpha_k = exp(H_kW_a^{\top}) / \sum_{i=1}^{K}(exp(H_iW_a^{\top}))
\end{equation}
where $W_a \in R^{1 \times 2d_h}$ denotes the learnable parameter matrix, and $i = 1, \ldots , K$. 
Based on this, the final node representation $H_{fin}$ is aggregated as follows:
\begin{equation}
    H_{fin} = \sum_{k=1}^{K} \alpha_{k}H_{k}
\end{equation}

\subsection{Analysis of Components}
Based on Definition \ref{graphDocument} and the analysis presented in graph serialization part, it is evident that graph documents can converge to the limiting distribution. Notably, this distribution tends to vary across different graphs. Consequently, we propose the following Lemma to examine the expressiveness of graph documents and prove it through graph isomorphism, which can be found in Appendix A.1.

\begin{lemma}
Graph documents possess the capability to distinguish non-isomorphic graphs.
\end{lemma}

Additionally, we also analyze the walk coverage of walk tokens mathematically and draw the conclusion that it is feasible to approximate full information coverage with a limited number of walk tokens. We prove that the probability of a particular information type being uncovered decreases at a rate faster than exponential decay. The detailed proof can be found in Appendix A.2.

The expressiveness of different token types is also analyzed in Appendix A.3. Our analysis concludes that SGPM-tokens can capture global information, Hop-tokens have the potential to match or exceed the performance of MPNNs, and Walk-tokens demonstrate superiority over hop-wise tokens, such as those in NAGphormer \cite{NAGphormer}. Lastly, we analyze the time and space complexity of Tokenphormer. Detailed analyses are provided in Appendix A.4.

\section{Experiments}
\label{Experiment}

\begin{figure*}[tb]
\centering
\includegraphics[width=0.78\textwidth]{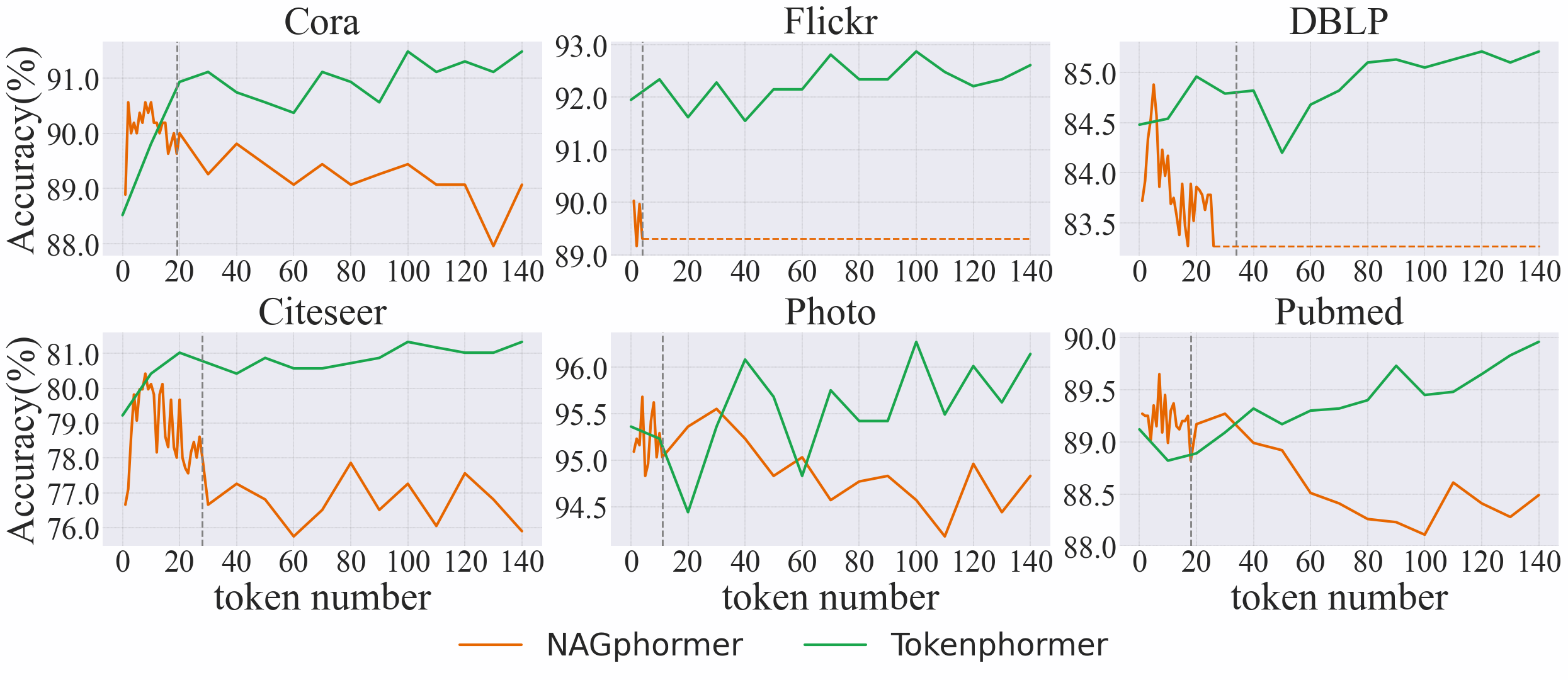}
\caption{Expressiveness Comparison.
The grey dashed line denotes graph diameter. For Flickr and DBLP, the orange dashed line means the NAN result.}
\label{allnagtp}
\end{figure*}

\subsection{Comparison of Tokenphormer with Baselines}
Table \ref{tab:main} compares all benchmark methods (Detailed descriptions of methods and datasets can be found in Appendix B.1) with Tokenphormer for node prediction tasks. Tokenphormer's results, presented as mean values with standard deviations from 10 runs with different random seeds, demonstrate state-of-the-art performance, outperforming most baselines on six datasets of varying sizes, proving the model's effectiveness.

Tokenphormer outperforms most MPNN baselines, showcasing strong generalization ability due to two factors: (1) Utilizing {\em hop-tokens} to enhance the density limit of {\em walk-tokens}, approximating layer-wise MPNN; (2) Capturing information beyond traditional MPNNs with diverse token types. As message passing layers increase, MPNNs suffer from over-smoothing and over-squashing, limiting their capacity to utilize information from distant neighbors.

Tokenphormer's superiority over graph Transformer baselines is due to several advantages: (1) {\em SGPM-tokens} extend the length limit of {\em walk-token}, capturing valuable contextual and global information for each node; (2) {\em Hop-tokens} extend the density limit of {\em walk-token}, retaining essential information from neighboring nodes and maintaining local relationships; (3) {\em Walk-tokens} generated through four types of walks add flexibility in capturing information from near and far neighbors. These mechanisms enable Tokenphormer to learn node representations through fine-grained tokens, capturing intricate details and relationships across the entire graph, leading to superior performance.

\subsection{Ablation Study}
The effectiveness of various token types was verified through ablation experiments conducted on all datasets. The results, detailed in Appendix B.3, show a decrease in accuracy with the removal of any token type.

Furthermore, we compared mixed walks of uniform random walks, non-backtracking random walks, neighborhood jump walks, and non-backtracking neighborhood jump walks with four single walks. The findings, which can be found in Appendix B.4, also demonstrate the superiority of mixed walks.

\subsection{Expressiveness Comparison}

Tokenphormer enhances model performance by utilizing a more extensive set of valuable tokens. 
To evaluate Tokenphormer's token expressive ability with a maximum number of tokens, we conducted a thorough quantitative and qualitative analysis, comparing NAGphormer~\cite{NAGphormer} and Tokenphormer under various token number settings.

For Tokenphormer, we use a step size of 10 {\em walk-tokens} to explore up to 140 {\em walk-tokens} with 1 fixed {\em SGPM-token} and 3 {\em hop-tokens}. For NAGphormer, we set step size as 1 for token numbers ranging from $1$ to graph diameter and 10 for other cases. Figure~\ref{allnagtp} shows the quantitative performance across six benchmark datasets. As the token number increases, NAGphormer's performance rises rapidly, peaks before the token number equals the graph diameter, and then declines, falling below Tokenphormer. This means that NAGphormer's Hop2Token strategy still suffers from some extent of over-smoothing with the token number's increase. Conversely, Tokenphormer improves in accuracy and becomes more stable with more tokens.


\subsection{Experiments on Heterogeneous Datasets}

Three more heterogeneous datasets, including Cornell, Wisconsin, and Actor, were also evaluated (Table~\ref{hetero}). Results show that even without SGPM, Tokenphormer is comparable to other methods and highly expressive on heterogeneous graphs. 1) GCN-based models perform poorly on heterogeneous graphs as they aggregate information from directly connected nodes. 2) Some Transformer-based models perform poorly, indicating susceptibility to irrelevant noise. 3) Tokenphormer performs remarkably well, especially better than NAGphormer, indicating its ability to explore graphs comprehensively and flexibly, thus obtaining richer information for target node representation learning.

\begin{table}[tbp]

    \centering
    \small
    \setlength{\tabcolsep}{1mm}{%
    \begin{tabular}{lcccc}
    \hline
    \textbf{Method} &\textbf{Year} &\textbf{Cornell}  &\textbf{Wisconsin} &\textbf{Actor}\\
    \hline
    GCN & 2017 & 45.67$\pm$7.96 & 52.55$\pm$4.27 &  28.73$\pm$1.17\\
    GAT & 2017 & 47.02$\pm$7.66 & 57.45$\pm$3.51 & 28.33$\pm$1.13 \\
    APPNP & 2018 & 41.35$\pm$7.15 & 55.29$\pm$3.90 & 29.42$\pm$0.81 \\
    GATv2 & 2022 & 50.27$\pm$8.97 & 52.74$\pm$3.96 & 28.79$\pm$1.47 \\
    \hline
    SAN & 2021 & 50.85$\pm$8.54 & 51.37$\pm$3.08 & 27.12$\pm$2.59 \\
    Gapformer & 2023 &  \textbf{77.57$\pm$3.43} & 83.53$\pm$3.42 & 36.90$\pm$0.82 \\
    NAGphormer & 2023 & 56.22$\pm$8.08 & 63.51$\pm$6.53 & 34.33$\pm$0.94\\
    \hline
    Tokenphormer & ours & 76.22$\pm$2.13 & \textbf{86.17$\pm$2.30} &  \textbf{37.01$\pm$0.83} \\
    \hline
    \end{tabular}
    }
    \caption{Experiments on Heterogeneous Datasets. The results of Tokenphormer here are obtained without SGPM-tokens. Best results are in \textbf{bold}. The experiment results of baselines are taken from~\cite{gapformer}.}
    \label{hetero}
\end{table}

\section{Conclusion}

This paper proposes Tokenphormer, a novel graph Transformer that overcomes the limitations of MPNNs and graph Transformers by generating fine-grained and comprehensive node tokens. Walk-tokens, created through diverse walks with varying walk tendencies and receptive fields, ensure fine-grained exploration of the graph. SGPM-tokens capture global information, extending the length limit of walk-tokens for greater comprehensiveness, while hop-tokens, generated by decoupling message-passing layers, enhance local coverage. Tokenphormer achieves state-of-the-art performance on node classification on diverse datasets, demonstrating its effectiveness for various graph tasks.

\section*{Acknowledgments}
This work was supported by the Science and Technology Development Fund Macau SAR (0003/2023/RIC, 0052/2023/RIA1, 0031/2022/A, 001/2024/SKL for SKL-IOTSC), the Research Grant of University of Macau (MYRG2022-00252-FST),  Shenzhen-Hong Kong-Macau Science and Technology Program Category C (SGDX20230821095159012), and Wuyi University joint Research Fund (2021WGALH14). This work was performed in part at SICC which is supported by SKL-IOTSC, University of Macau. 


    \newpage
    \appendix
    \begin{center}
    \LARGE
        Appendix
    \end{center}
    
    
    
    \section{Theoretical Analysis of Tokenphormer}\label{AnalysisTokenphormer}
    In this section, we conduct various analyses to evaluate the effectiveness and expressiveness of Tokenphormer. 
    Through these analyses, we aim to gain a deeper understanding of the capabilities and performance of Tokenphormer in effectively representing graph data.
    
    \subsection{Analysis of Graph Documents}\label{gdAnalysis}
    
    In main content, we propose the following Lemma to examine the expressiveness of graph documents.
    
    \begin{lemma}
    Graph documents possess the capability to distinguish non-isomorphic graphs.
    \end{lemma}
    \begin{proof}
    \label{graphdocumentDis}
    Suppose $G$ is an undirected, weighted, and connected graph. If there exists an edge between two arbitrary nodes, denote as $v_i$ and $v_j$, then its weight $w_{ij}=w_{ji}>0$, else $w_{ij}=w_{ji}=0$. The transition probability of random walk from $v_i$ to $v_j$ is:
    \begin{equation}
        P_{ij} = \frac{w_{ij}}{\sum_k{w_{ik}}}
    \end{equation}
    
    Since random walk on $G$ can be viewed as an invertible Markov chain, we have:
    \begin{equation}
        \pi (v_i)P_{ij} = \pi (v_j)P_{ji} \rightarrow \frac{\pi (v_i)}{\sum_k w_{ik}} = \frac{\pi (v_j)}{\sum_k w_{jk}}
    \end{equation}
    
    Then, for all node $v_i$, we have:
    \begin{equation}
        \frac{\pi (v_i)}{\sum_k w_{ik}} = C \rightarrow \pi (v_i) = C \sum_k w_{ik}
    \end{equation}
    where $C$ is a constant. Since $\pi$ is a probability distribution on all states, we have $\sum_i \pi (v_i) = 1$. Then we can get:
    \begin{equation}
        \sum_i (C \sum_k w_{ik}) = 1 \rightarrow C = \frac{1}{\sum_i  \sum_k w_{ik}}
    \end{equation}
    
    Finally, the limiting distribution of the Markov Chain is calculated:
    \begin{equation}
        \pi (v_i) = \frac{\sum_k w_{ik}}{\sum_i  \sum_k w_{ik}}
    \end{equation}
    
    Alternatively, in the scenario where $G$ is unweighted, i.e., the weight between two connected nodes is 1, the resulting limiting distribution of the graph document can be described as follows:
    \begin{equation}
        \pi (v_i) = \frac{\sum_k w_{ik}}{\sum_i  \sum_k w_{ik}} = \frac{d_i}{\sum_i d_i} = \frac{d_i}{2|E|}
    \end{equation}
    
    Let $G_a$ and $G_b$ represent two arbitrary graphs. If they are isomorphic, the limiting distribution of these two graphs must equal each other. In other words, set $\pi_{ai}$ ($i \in V_{G_a}$) and set $\pi_{bj}$ ($j\in V_{G_b}$) are bijective to each other. In this case, the graph document ensures that it would not distinguish isomorphic graphs into different graphs. 
    
    If $G_a$ and $G_b$ are non-isomorphic, there are two situations. For the first situation, if $G_a$ differs from $G_b$ in $|V|$, it is clear that $\pi_a \ne \pi_b$ since the larger set between $\pi_{ai}$ and $\pi_{bj}$ is not surjective to smaller one. Thus, the graph document can easily distinguish them. For another situation, when $G_a$ and $G_b$ differs in $|V|$, $G_a$ and $G_b$ are indistinguishable to graph document if and only if $G_a$ and $G_b$ satisfy: $f(x) = x$ is bijective function when:
    \begin{equation}
    \small
        D = \left\{ \frac{\sum_k w_{aik}}{\sum_{ai}  \sum_k w_{aik}} \Big\vert i\in V_{G_a}\right\}, 
        R = \left\{ \frac{\sum_k w_{bjk}}{\sum_{bj}  \sum_k w_{bjk}} \Big\vert j \in V_{G_b} \right\}
    \end{equation}
    and
    \begin{equation}
    \small
        D = \left\{ \frac{\sum_k w_{bjk}}{\sum_{bj}  \sum_k w_{bjk}} \Big\vert j \in V_{G_b} \right\}, 
        R = \left\{ \frac{\sum_k w_{aik}}{\sum_{ai}  \sum_k w_{aik}} \Big\vert i\in V_{G_a}\right\}
    \end{equation}
    where $D$ is the domain and $R$ represents the range of rejection function $f(x) = x$. In this case, the probability for two graphs with the same $|V|$ and different structures to be indistinguishable drops rapidly with the increase of $|V|$.
    \end{proof}
    
    Under this analysis, the distinguishing ability of the graph document enables it to capture unique characteristics and structures inherent to the graph, thus obtaining high expressive power.

    
    \subsection{Analysis of Walk Coverage}\label{Walk Coverage}
    
    Walk coverage pertains to the extent of information coverage within a node's neighborhood, utilizing a limited number of walks. We classify walks based on node labels to achieve this, enabling a systematic analysis of the information encapsulated within the walks. 
    
    When two walks have the same length ($k$) and share identical node labels at corresponding positions, we assume the information obtained from these walks is similar. Consequently, we classify such walks into the same information type, denoted as ${ l_1, l_2, ..., l_k }$, where $l_i$ represents the label of node $i$. By doing so, we can analyze walk coverage based on the coverage of information types present in a node's neighborhood.
    
    \begin{figure}[!t] 
    \centering
    \includegraphics[width=0.42\textwidth]{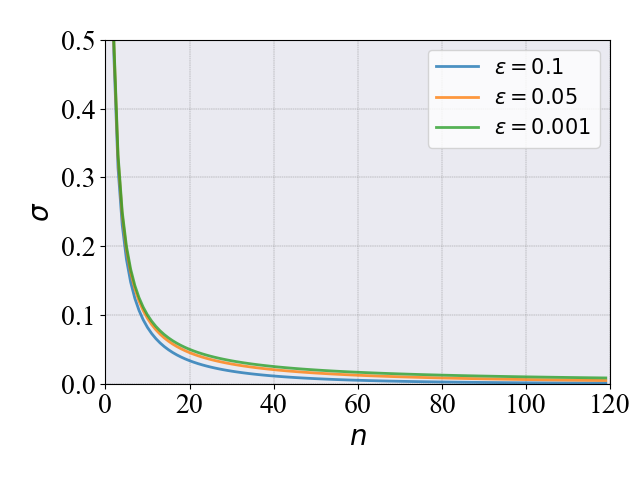}
    \caption{\textbf{Coverage Analysis.} Let $\sigma = \frac{\exp\left(-2\epsilon^2n\right)}{n}$, the figure shows the change of $\sigma$ with increase of $n$ in case of different $\epsilon$ in equation \ref{equCoverage}.}
    \label{coverageanaFig}
    \end{figure}
    
    
    
    As proven by~\cite{sun2022beyond}, we follow a similar approach to analyzing the walk coverage. The definition of label-limited transition matrix $P^{(l)}$ is:
    \begin{equation}
        P_{i,j}^{(l)}=
        \begin{cases}
            P_{i,j}, & Y_j = l \\
            0, & \text{otherwise}
        \end{cases}
    \end{equation}
    where $l$ denotes the specified label, $Y_j$ is the label of $v_j$ and $P_{i,j}$ is the raw transition probability of $v_i$ to $v_j$. Then the probability of $v_i$ transition to $v_j$ through the type of $\{ l_1, l_2, ..., l_k \}$ is:
    \begin{equation}
        P^{ \{ l_1, l_2, \ldots, l_k \}} = \prod_{i=1}^{k} P^{(l_i)}
    \end{equation}
    
    Then, the probability of an information type $\{ l_1, l_2, \ldots, l_k \}$ starting from $v_S$ being sampled is:
    \begin{equation}
        P_{S}^{\{l_1, l_2, \ldots, l_k\}} = \sum_{i=1}^{|V|} P_{S,i}^{\{ l_1, l_2, \ldots, l_k \}}
    \end{equation}
    
    Let $N_{S}^{\{ l_1, l_2, \ldots, l_k \}}$ denotes number of sampled information types in $n$ times sampling. It is clear that $N_{S}^{\{ l_1, l_2, \ldots, l_k \}}$ follows a binominal distribution:
    \begin{equation}
        N_{S}^{\{l_1, l_2, \ldots, l_k\}} \sim B(n, P_{S}^{\{l_1, l_2, \ldots, l_k\}})
    \end{equation}
    
    At last, according to Hoeffding's inequality~\cite{Hoeffding}, we have:
    \begin{equation}
        \label{equCoverage}
        \begin{aligned}
            D_{S}^{\{l_1, l_2, \ldots, l_k\}} = |\frac {N_{S}^{\{l_1, l_2, \ldots, l_k\}}}{n} - P_{S}^{\{l_1, l_2, \ldots, l_k\}}|\\
            P(D_{S}^{\{l_1, l_2, \ldots, l_k\}} > \epsilon) \leq \frac{\exp\left(-2\epsilon^2n\right)}{n}
        \end{aligned}
    \end{equation}
    where $D_{S}^{\{l_1, l_2, \ldots, l_k\}}$ stands for sampling deviation. Thus, for any $\epsilon > 0$, the probability of one information type being not sampled being larger than $\epsilon$ decreases rapidly (since $\exp\left(-2\epsilon^2n\right)$ already decreases exponentially, and $\frac{\exp\left(-2\epsilon^2n\right)}{n} \leq \exp\left(-2\epsilon^2n\right)$ when $n \ge 1$) as the number of sampled walks $n$ increases. (As shown in Fig. \ref{coverageanaFig})
    
    The above analysis demonstrates that, although covering all information from the node's neighborhood is challenging, obtaining most information is still achievable through a limited number of samplings. Our experiments also verified this point: when the number of walks increased to a certain degree, the experimental results tended to be stable, and the model reached the fitting.

    \subsection{Analysis of tokens}\label{analysistokens}
    The proposed Tokenphormer generates multiple tokens with different granularities and learns each other's weights in a fully connected scheme through the Transformer, thus alleviating the problems of over-smoothing and over-squashing. At the same time, by generating plentiful walk-based tokens, Tokenphormer is more effective in retaining structure information than other graph Transformers.
    
    \stitle{SGPM-token} Appendix~\ref{gdAnalysis} has demonstrated that the graph document can distinguish between different graphs, showcasing its proficiency in capturing the global information of the graph. The the SGPM-token, derived through SGPM, is critical in acquiring contextual information for each node within the graph document. With the generation of numerous walks for each node, the graph document ensures comprehensive coverage for every node. Additionally, the length of the graph sentence follows a normal distribution, with a mean equal to the graph's radius. This intentional design choice aligns the receptive field of the graph sentence, starting from each node, with the global characteristics of the graph. Furthermore, the graph sentences enable the exploration of far nodes by utilizing the non-backtracking random walk. This approach effectively reduces redundancy, allowing nodes to acquire diverse contexts and enriching the semantic information they gather~\cite{NEURIPS2022_RFGNN}.


    \stitle{Hop-token} \label{analysis_hop}
    The hop-token comes from the decoupled message passing layer, and its calculation is:
    \begin{equation}
        H^k = A^k X
    \end{equation}
    where $H^k$ represents the feature matrix of $k$-hop. $H_i^k$ is node $i$'s aggregated information from $k$-th layer of MPNN. {\em hop-tokens} in Tokenphormer can be viewed as MPNN with layer-wise attention added. Since MPNN already proves its expressive power in node prediction, the hop-token has the potential to achieve equal or better performance~\cite{NAGphormer}.
    
    \begin{figure}[!t]
    \centering
    \includegraphics[width=0.385\textwidth]{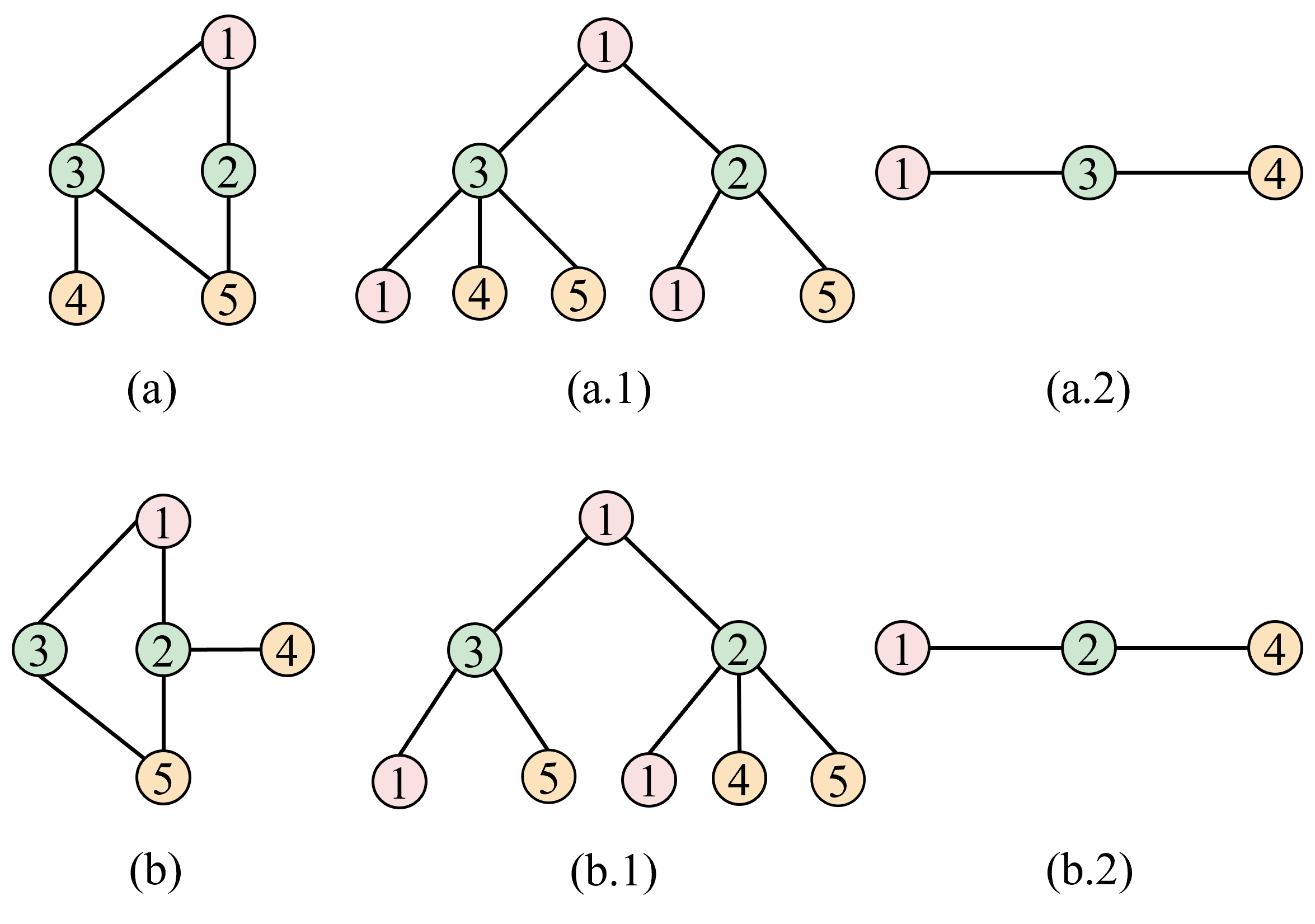}
    \caption{hop-token vs walk-token.}
    \label{hopwalk}
    \end{figure}
    
    \stitle{Walk-token}
    Instances arise where the hop-token falls short in distinguishing between certain situations. This is exemplified in Fig. \ref{hopwalk}, where (a) and (b) represent graphs with distinct structures, and (a.1) and (b.2) illustrate the message passing trees for node 1 in their respective graphs.
    In the depicted figure, the horizontal layers can be interpreted as hop-tokens, while the vertical paths represent walk-tokens. It is evident that when considering a 2-hop message passing scenario, hop-tokens cannot differentiate between graphs (a) and (b). However, walk-tokens, exemplified by (a.2) and (b.2), exhibit the potential to distinguish between the two graphs successfully.

    \subsection{Complexity Analysis}\label{complexityAnalysis}
    The generation of all tokens in Tokenphormer can be performed beforehand. SGPM is pre-trained, so the SGPM-token can be obtained by Tokenphormer easily in $O(1)$ complexity. Moreover, walks are generated before training, and Hop-tokens can also be computed in advance. So, we only consider the complexity of Tokenphormer without the token generation phase.
    
    \stitle{Time complexity} The time complexity of Tokenphormer primarily relies on the self-attention module of the Transformer. 
    For each node in the graph, we perform computations proportional to the square of the number of tokens $N_t$ and the dimension of the feature $d_F$, thus resulting in a time complexity of $O(|V| \cdot N_t^2 \cdot d_F)$, where $|V|$ denotes the number of nodes.
    
    \stitle{Space complexity} The space complexity is determined by the number of model parameters and the outputs of each layer. It can be broken down into two main parts: 1) The Transformer layer contributes $O(d_F^2 \cdot N_L)$, where $N_L$ is the number of Transformer layers. 2) The attention matrix and hidden node representations add $O(S_b \cdot N_t^2 + S_b \cdot N_t \cdot d_F)$, where $S_b$ denotes the batch size. Thus, the total space complexity is $O(d_F^2 \cdot N_L + S_b \cdot N_t^2 + S_b \cdot N_t \cdot d_F)$.

    \section{Experiment}
    \subsection{Experiment Setup}
    \stitle{Datasets}
    To comprehensively evaluate the effectiveness of Tokenphormer, we conduct experiments on six homogeneous graph datasets with varying node numbers, ranging from 2,708 to 19,717. The datasets include small-sized datasets like Cora and Citeseer (around 3,000 nodes), medium-sized datasets like Flickr and Amazon Photo (around 7,500 nodes), and large-sized datasets like DBLP and Pubmed (around 18,000 nodes). We apply a 60\%/20\%/20\% train/val/test split (consistent with Gophormer~\cite{Gophormer} and NAGphormer~\cite{NAGphormer}) for these benchmark datasets. 
    The detailed dataset information is presented in Table \ref{tab1}. 
    
    \begin{table}[!hbt]
    \centering
    \small
    \setlength{\tabcolsep}{1mm}{%
    \begin{tabular}{lccccc}
    \hline
    \textbf{Dataset}  & \textbf{Nodes} & \textbf{Edges} & \textbf{Classes} & \textbf{Features} & \textbf{Diameter} \\
    \hline
    Cora        & 2,708    & 5,278    & 7         & 1,433      & 19      \\
    Citeseer    & 3,327    & 4,522    & 6         & 3,703      & 28      \\
    Flickr      & 7,575    & 239,738  & 9         & 12,047     & 4      \\
    Photo       & 7,650    & 238,163  & 8         & 745        & 11      \\
    DBLP        & 17,716   & 52,864   & 4         & 1,639      & 34      \\
    Pubmed      & 19,717   & 44,324   & 3         & 500        & 18      \\
    \hline
    \end{tabular}
    }
    \caption{Statistics on datasets. Graph diameter is calculated using the method described in~\cite{graphDiameter}.}
    \label{tab1}
    \end{table}

    \stitle{Baselines} To comprehensively assess the effectiveness of Tokenphormer in the context of node classification tasks, a comparative analysis is conducted against a diverse set of 16 advanced baseline models. The evaluated models encompass eight prominent MPNN methods, including GCN~\cite{GCN_kipf2017semi}, GAT~\cite{GAT_veličković2018graph}, GraphSage~\cite{GraphSAGE_Hamilton2017InductiveRL}, APPNP~\cite{Klicpera_Bojchevski_Günnemann_2018}, JKNet~\cite{MPNN1_pmlr-v80-xu18c}, GPR-GNN~\cite{GPR-GNN}, GRAND+~\cite{GRAND+}, GATv2~\cite{GATv2} and eight innovative graph Transformer models, including GT~\cite{GT}, Graphormer~\cite{Graphormer}, SAN~\cite{Kreuzer_Beaini_Hamilton_Létourneau_Tossou_2021}, Gophormer~\cite{Gophormer}, ANS-GT~\cite{ANS-GT}, GraphGPS~\cite{rampasek2022GPS}, Exphormer~\cite{Exphormer}, Gapformer~\cite{gapformer}, and NAGphormer ~\cite{NAGphormer}. 
    
    \begin{table*}[htb!]
        \begin{center}
        \small
        \setlength{\tabcolsep}{1mm}{%
        \begin{tabular}{lcccccc}
        \hline
        \textbf{Method} &\textbf{Cora} &\textbf{Citeseer} &\textbf{Flickr} &\textbf{Photo} &\textbf{DBLP} &\textbf{Pubmed}\\
        \hline
        Tokenphormer&  \textbf{91.20$\pm$0.47} &  \textbf{81.04$\pm$0.24} &  \textbf{92.44$\pm$0.35} &  \textbf{96.14$\pm$0.14} &  \textbf{85.13$\pm$0.10} &  \textbf{89.94$\pm$0.20}\\
        \hline
        w/o SGPM-token & 90.33$\pm$0.30 & 79.85$\pm$1.00 & 91.75$\pm$0.50 & 95.61$\pm$0.47 & 84.42$\pm$0.41 & 89.38$\pm$0.46\\
        w/o walk-token & 88.37$\pm$1.23 & 79.70$\pm$0.71 & 92.24$\pm$0.35 & 94.81$\pm$0.44 & 83.93$\pm$0.16 & 89.60$\pm$0.35\\
        w/o hop-token& 90.73$\pm$0.34 & 80.19$\pm$2.32 & 91.97$\pm$0.71 & 95.80$\pm$0.30 & 85.01$\pm$0.16 & 89.35$\pm$0.36\\
        \hline
        only URW & 90.41$\pm$0.56 & 80.69$\pm$0.49 & 92.07$\pm$0.77 & 95.83$\pm$0.32 & 84.90$\pm$0.39 & 89.33$\pm$0.39\\
        only NBRW & 89.72$\pm$0.36 & 80.03$\pm$0.90 & 92.27$\pm$0.53 & 95.74$\pm$0.41 & 84.86$\pm$0.25 & 89.26$\pm$0.45\\
        only NJW & 90.80$\pm$0.66 & 80.92$\pm$0.50 & 92.38$\pm$0.70 & 95.61$\pm$0.38 & 84.70$\pm$0.44 & 89.50$\pm$0.36\\
        only NBNJW & 89.63$\pm$0.64 & 80.56$\pm$0.74 & 92.09$\pm$0.39 & 96.00$\pm$0.31 & 84.71$\pm$0.35 & 89.36$\pm$0.35\\
        \hline
        \end{tabular}
        }
        \end{center}
        \caption{Ablation Experiment. Best results are in \textbf{bold}. URW, NBRW, NJW, and NBNJW represent uniform random walk, non-backtracking random walk, neighborhood jump walk, and non-backtracking neighborhood jump walk. 
        }
        \label{tab:abla}
    \end{table*}
    
    \subsection{Experiment Details}
    
    For the model configuration of Tokenphormer, in SGPM, we generate 100 non-backtracking random walks for each node on the graph as the training dataset, where the walk lengths follow a normal distribution $\mathcal{N}(\mu= $ graph radius or 10$, \sigma^2=1)$, and 20 walks are generated as the validate dataset. In the {\em walk-token} part, we generate 100 walks with lengths starting from 4 and a ratio of mixed walk types starting from $25\%:25\%:25\%:25\%$ and tune these hyper-parameters to the best. For {\em hop-token}, we fixed the hop number to 3 for all datasets. In model training, we adopt the AdamW optimizer~\cite{AdamW} and set the learning rate to $1e^{-2}$ for Flickr and Photo and $5e^{-3}$ for other datasets and weight decay to $1e^{-5}$. The dropout rate is set to $0.1$, and the Transformer head is set to 1. The batch size is set to 2000. For SGPM, as a pre-train phase, we conducted it on a Linux server with DGX-H800 (4 $\times$ NVIDIA H800 80 GB) GPU. Tokenphormer was conducted on a Linux server with 1 R9-5950X CPU and an NVIDIA RTX 3090TI (24GB) GPU.

    \subsection{Ablation Study}\label{app:ablationStudy}
    To analyze the effectiveness of {\em SGPM-token}, {\em walk-token}, and {\em hop-token}, ablation experiments are carried out on all datasets, as shown in Table \ref{tab:abla}. It can be seen that with any token dropped, the accuracy decreases to some extent.
    
    \stitle{SGPM-token} {\em SGPM-token} captures contextual and global information of the selected node. Without {\em SGPM-token}, the accuracy of Tokenphormer drops in all six datasets.
    
    \stitle{Walk-token} Tokenphormer's accuracy decreases the most when {\em walk-token} are omitted. This is because {\em walk-tokens} are able to capture fine-grained information from a large neighborhood.
    
    \stitle{Hop-token} {\em Hop-token} ensures coverage in $k$-hop neighborhoods. However, for most datasets, large number of {\em walk-tokens} already cover all nodes in $k$-hop neighborhoods, resulting in a smaller accuracy decrease when {\em hop-token} is withdrawn.
    
    \subsection{Walk Comparison}\label{app:walkComparison}
    We compare mixed walks (uniform random walks, non-backtracking random walks, neighborhood jump walks, and non-backtracking neighborhood jump walks) with four single walks. Table \ref{tab:abla} shows that different walks vary in effectiveness across datasets, but mixed walks consistently achieve higher accuracy. For example, for citation networks like Cora and Citeseer, where information from near neighbors is crucial, walks without non-backtracking property can perform better than non-backtracking ones. For datasets like Photo, the situation changes. The non-backtracking neighborhood jump walk performs the best among all types of walks. These results indicate that mixed walks leverage the strengths of all types to enhance overall accuracy.

\end{document}